\documentclass[11pt,a4paper]{article}
\usepackage[T1]{fontenc}
\usepackage[utf8]{inputenc}
\usepackage{graphicx}
\usepackage{amsmath,amssymb,amsthm,amsfonts}
\usepackage{geometry}
\usepackage{booktabs}
\usepackage{hyperref}

\newtheorem{theorem}{Theorem}
\newtheorem{proposition}{Proposition}
\newtheorem{lemma}{Lemma}
\newtheorem{corollary}{Corollary}
\theoremstyle{definition}
\newtheorem{definition}{Definition}
\newtheorem{remark}{Remark}

\title{\textbf{Asymptotic Signal Subspace Recovery in Softmax Attention Models}}
\author{\textbf{Lan V. Truong}\\
Faculty of Computer Science and Engineering\\
Ho Chi Minh City University of Technology (HCMUT)\\
Vietnam National University Ho Chi Minh City (VNU-HCM), Vietnam\\
\texttt{lantv@hcmut.edu.vn}}

\date{}

\begin{document}

\maketitle

\begin{abstract}
Attention mechanisms have demonstrated remarkable empirical success in identifying relevant information from large collections of tokens, yet the theoretical principles underlying this behavior remain poorly understood. We study a stylized softmax-attention model in which a query vector is learned by stochastic gradient ascent from a collection of informative and nuisance tokens. Exploiting the symmetry of the model, we derive a population objective and characterize the limiting ordinary differential equation governing the learning dynamics. Using tools from stochastic approximation and dynamical systems theory, we establish a rigorous connection between the stochastic learning algorithm and its deterministic limit. Our main result shows that, under suitable high-dimensional scaling assumptions and standard step-size conditions, the learned query converges almost surely to the one-dimensional signal subspace spanned by the latent informative direction. Equivalently, the query asymptotically recovers the latent signal up to the intrinsic sign ambiguity. These results provide a rigorous theoretical foundation for understanding attention mechanisms as signal extraction procedures in high-dimensional noisy environments and offer a dynamical-systems perspective on how attention discovers relevant information in the presence of substantial noise.
\end{abstract}

\tableofcontents

\newpage

\section{Introduction}

The Transformer architecture has emerged as the dominant paradigm in modern machine learning, serving as the foundation of large language models, vision transformers, and numerous multimodal systems. At the heart of the Transformer lies the attention mechanism, which dynamically assigns weights to input tokens and aggregates information according to their relevance. Despite its remarkable empirical success, the theoretical principles underlying the behavior of attention remain only partially understood.

Existing theoretical studies of attention have primarily focused on its expressive power, optimization landscape, approximation properties, or statistical generalization behavior. In most of these analyses, the query, key, and value representations are assumed to be fixed or already learned. Consequently, relatively little is known about the fundamental mechanism by which attention identifies relevant information in a collection of noisy tokens. In particular, a basic scientific question remains largely unanswered:

\begin{quote}
Can attention itself discover a hidden signal embedded among many noisy tokens, and if so, what mathematical mechanism drives this discovery process?
\end{quote}

Recent theoretical works have sought to explain how attention mechanisms identify relevant information in large collections of tokens. Existing analyses have primarily focused on the expressive power of self-attention and transformers \cite{yun2020transformers,likhosherstov2021expressive,perez2019turing}, their interpretation as associative memory and retrieval mechanisms \cite{ramsauer2021hopfield}, and their ability to perform in-context learning \cite{garg2022can,vonoswald2023transformers,Cao2026TransformersLR}. More closely related to our setting, Barnfield et al.~\cite{barnfield2025sparse} studied a sparse-token attention model and showed that gradient-based training induces nontrivial alignment between the query vector and a hidden signal direction. In contrast, we analyze the full stochastic approximation dynamics of the query update and establish almost sure convergence to the latent signal subspace.

This paper investigates these questions through a sparse-signal detection framework. We consider a collection of tokens
\begin{align}
x_i
=
v_i \theta_d \xi_d + z_i,
\qquad
i=1,\ldots,L_d,
\end{align}
where
\begin{itemize}
\item $\xi_d \in S^{d-1}$ is an unknown signal direction;
\item $\theta_d>0$ denotes the signal strength;
\item $v_i\in\{0,1\}$ indicates whether the $i$-th token is informative;
\item $z_i\sim N(0,I_d)$ is an isotropic Gaussian noise vector.
\end{itemize}

Under this model, only a sparse subset of tokens contains useful information aligned with the latent signal direction, while the remaining tokens are pure noise.

To understand how attention identifies informative tokens, we study the attention map
\begin{align}
f_q(X)
=
\sum_{i=1}^{L_d}
a_i(q,X)x_i,
\end{align}
where
\begin{align}
a_i(q,X)
=
\frac{\exp\!\bigl(\beta\langle q,x_i\rangle\bigr)}
{\sum_{j=1}^{L_d}
\exp\!\bigl(\beta\langle q,x_j\rangle\bigr)}.
\end{align}

Here $q\in S^{d-1}$ denotes the query vector and $\beta>0$ is the inverse-temperature parameter.

The key quantity governing the behavior of attention is the alignment
\begin{align}
\rho
=
\langle q,\xi_d\rangle.
\end{align}

When $\rho$ is positive, informative tokens tend to produce larger logits than non-informative tokens. Consequently, they receive larger attention weights and contribute more strongly to the attention output. This creates a positive-feedback mechanism: improved alignment increases the attention placed on informative tokens, which in turn further reinforces alignment with the latent signal direction.

Our analysis formalizes this intuition and establishes a connection between self-attention and sparse-signal detection. Exploiting the rotational symmetry of the model, we show that the optimization landscape can be reduced to a one-dimensional problem governed by the alignment parameter $\rho$. This reduction enables a detailed study of the geometry of the objective function and the resulting alignment dynamics.

Our main result shows that this positive-feedback mechanism is sufficiently strong to recover the latent signal direction. Under suitable high-dimensional scaling assumptions and standard stochastic approximation conditions, the learned query vector converges almost surely to the signal subspace generated by the latent direction $\xi_d$. Consequently, attention-based optimization asymptotically separates informative tokens from nuisance tokens and recovers the underlying signal up to the unavoidable sign ambiguity.

More broadly, our work suggests a new interpretation of attention. Rather than acting solely as an information aggregation operator, attention can also be viewed as a statistical procedure for discovering latent structure hidden within high-dimensional data. This perspective establishes a bridge between Transformer architectures, high-dimensional probability, sparse-signal detection, and stochastic approximation theory.

The main contributions of this work are summarized as follows.

\begin{enumerate}

\item \textbf{A tractable model for attention-based signal recovery.}

We introduce a stylized softmax-attention model consisting of informative and nuisance tokens and study the learning dynamics of a query vector under stochastic gradient ascent. The model captures the essential signal-selection mechanism of attention while remaining amenable to rigorous mathematical analysis.

\item \textbf{Characterization of the limiting attention dynamics.}

We derive the population objective, the associated population vector field, and the corresponding limiting ordinary differential equation governing the asymptotic mean-field evolution of the query vector.

\item \textbf{A rigorous stochastic approximation analysis.}

Using the ODE method and asymptotic pseudo-trajectory theory, we establish that the stochastic query-learning recursion tracks the solutions of the limiting dynamical system, thereby connecting the stochastic optimization algorithm with its deterministic counterpart.

\item \textbf{Global convergence and signal subspace recovery.}

Our main theorem establishes that, under suitable scaling assumptions and standard stochastic approximation conditions,
\begin{align}
\lim_{k\to\infty}
\inf_{q^*\in\{\pm\xi_d\}}
\|q_k-q^*\|
=
0,
\qquad \mathrm{a.s.}
\end{align}

Consequently, the attention-learning dynamics recover the latent informative direction up to the intrinsic sign ambiguity and provide a rigorous explanation for attention-based relevance discovery in high-dimensional noisy environments.

\end{enumerate}

\section{Related Work}

\subsection{Attention Mechanisms and Transformer Theory}

Attention mechanisms were first introduced in neural machine translation
\cite{bahdanau2015neural} and later became the central component of transformer architectures
\cite{vaswani2017attention}. Their success has motivated extensive theoretical investigations into the expressive power and computational capabilities of attention-based models
\cite{yun2020transformers,perez2019turing,likhosherstov2021expressive}.
Recent works have also interpreted attention as a form of associative memory and information retrieval
\cite{ramsauer2021hopfield}, and have analyzed its role in in-context learning  
\cite{garg2022can,vonoswald2023transformers, Cao2026TransformersLR}.
In contrast, our work studies the optimization dynamics of a softmax-attention objective and asks whether gradient-based learning can recover an underlying latent signal direction from a mixture of informative and nuisance tokens.

\subsection{Stochastic Approximation and Dynamical Systems}

Our analysis is rooted in the classical stochastic approximation framework initiated by Robbins and Monro \cite{robbins1951stochastic}. The asymptotic behavior of stochastic recursive algorithms is commonly studied through their associated limiting ordinary differential equations (ODEs), an approach developed by Ljung \cite{ljung1977analysis}, Kushner and Clark \cite{kushner1978}, Kushner and Yin \cite{kushner2003stochastic}, and Borkar \cite{borkar2023stochastic}. A major advance in this direction is the dynamical-systems framework of Bena\"im \cite{benaim1996dynamical,benaim1999dynamics}, which interprets stochastic approximation trajectories as asymptotic pseudo-trajectories of deterministic flows and characterizes their limit sets through chain-recurrence theory. Extensions to differential inclusions were subsequently developed in \cite{benaim2005stochastic}. We employ these tools to analyze the long-term behavior of the proposed attention-learning dynamics and establish convergence to the latent signal subspace.

\subsection{Latent Signal Recovery in High Dimensions}

Recovering latent low-dimensional structure from high-dimensional observations is a central problem in statistics and machine learning. Classical examples include principal component analysis \cite{jolliffe2002principal,johnstone2001distribution}, factor models \cite{bai2003inferential}, spectral estimation and matrix factorization \cite{udell2016generalized}, and tensor decomposition \cite{kolda2009tensor,anandkumar2014tensor}. In these settings, recovery is typically achieved only up to sign or rotational ambiguity. Our result exhibits a similar phenomenon: the learned query vector converges to the latent signal direction up to sign. Unlike classical spectral approaches, however, the recovery mechanism studied here arises from softmax-attention optimization and gradient-based learning dynamics rather than eigendecomposition or moment-based estimation.

\subsection{Theoretical Understanding of Attention as Signal Extraction}

Recent theoretical works have sought to explain how attention mechanisms identify relevant information in large collections of tokens. Existing analyses have primarily focused on the expressive power of self-attention and transformers \cite{yun2020transformers,likhosherstov2021expressive,perez2019turing}, their interpretation as associative memory and retrieval mechanisms \cite{ramsauer2021hopfield}, and their ability to perform in-context learning \cite{garg2022can, vonoswald2023transformers}. More closely related to our setting, Barnfield et al.~\cite{barnfield2025sparse} studied a sparse-token attention model and showed that gradient-based training induces nontrivial alignment between the query vector and a hidden signal direction. In contrast, we analyze the full stochastic approximation dynamics of the query update and prove almost sure convergence to the latent signal subspace. To the best of our knowledge, this provides one of the first rigorous asymptotic signal-recovery results for a softmax-attention learning model.

\section{System Architecture and Data Modeling}
Let $d \in \mathbb{N}$ denote the dimension of the embedding space. Let $\xi_{d}$ be a fixed, hidden target directional unit vector residing on the unit hypersphere $S^{d-1} = \{ v \in \mathbb{R}^d : \|v\| = 1 \}$.

For a given context instance, the input data sequence is collected in a token data matrix $X \in \mathbb{R}^{L \times d}$ comprising $L$ discrete token rows $[x_1, x_2, \dots, x_L]^\top$. The individual tokens are structurally generated via a hidden mixture distribution:
\begin{align}
x_{i} = v_{i}\theta_{d}\xi_{d} + z_{i}, \quad \forall i \in \{1, 2, \dots, L\}
\end{align}
where the constituent variables are defined as follows:
\begin{itemize}
    \item $v_{i} \in \{0,1\}$ is a deterministic partition assignment indicator. The context window contains exactly $R$ informative tokens and $L-R$ pure noise tokens. We define the informative index subset as $\mathcal{I} = \{i : v_i = 1\}$ with $|\mathcal{I}| = R$, and the background noise index subset as $\mathcal{N}_o = \{j : v_j = 0\}$ with $|\mathcal{N}_o| = L - R$.
    \item $\theta_d > 0$ represents the signal scaling magnitude.
    \item $z_{i} \in \mathbb{R}^d$ are independent and identically distributed (i.i.d.) random background noise vectors obeying a standard multivariate isotropic Gaussian distribution:
    \begin{equation}
    z_{i} \sim \mathcal{N}(0, I_{d}).
    \end{equation}
\end{itemize}

\subsection{Attention Weights and Field Expressions}
Let $q \in S^{d-1}$ represent the tracking query vector. The matching attention coefficient mapping $a_{i}(q, X)$ is modeled using the standard softmax function:
\begin{align}
a_{i}(q, X) = \frac{e^{\beta \langle x_{i}, q \rangle}}{\sum_{j=1}^{L} e^{\beta \langle x_{j}, q \rangle}},
\end{align}
where $\beta > 0$ acts as the inverse softmax temperature parameter. By construction, the coefficients are strictly bounded and represent a valid probability distribution across the tokens:
\begin{align}
0 < a_{i}(q, X) < 1, \quad \sum_{i=1}^{L_d} a_{i}(q, X) = 1.
\end{align}
The output representation vector generated by the pooling layer, denoted $f_{q}(X)$, is the convex combination of the input token rows:
\begin{align}
f_{q}(X) = \sum_{i=1}^{L_d} a_{i}(q, X) x_{i}.
\end{align}

\subsection{Unsupervised Objective and Algorithmic Tracking Dynamics}
Because $\xi_d$ is hidden from the algorithm, the system cannot directly evaluate or optimize the inner projection field $\langle f_q(X), \xi_d \rangle$. Instead, we optimize a self-consistent unsupervised energy objective. We maximize the expected squared Euclidean norm of the pooled output:
\begin{align}
J_{d}(q) = \mathbb{E}\left[ \|f_{q}(X)\|^2 \right].
\end{align}
The true population objective function gradient is denoted by $\nabla_q J_d(q)$. At each discrete iteration step $k \in \mathbb{N}_0$, the algorithm draws a fresh, independent input matrix sequence sample $X_{k+1}$ and evaluates the exact empirical stochastic gradient vector $g_t \in \mathbb{R}^d$:
\begin{align}
g_k = \nabla_{q} \|f_{q}(X_{k+1})\|^2 \Big|_{q = q_k}.
\end{align}
The updating tracking query trajectory sequence is generated via a projected spherical stochastic gradient descent loop:
\begin{align}
q_{k+1} = \text{Proj}_{S^{d-1}}(q_k + \eta_k g_k) = \frac{q_k + \eta_k g_k}{\|q_k + \eta_k g_k\|}.
\end{align}
The schedule step sizes $\eta_k$ are configured to satisfy specific conditions linked to the dimensions to balance tracking speed against variance bounds:
\begin{align}
\eta_k = \frac{\gamma_k}{d^2}, \quad \text{where } \gamma_k > 0, \quad \sum_{k=0}^{\infty} \gamma_k = \infty, \quad \sum_{k=0}^{\infty} \gamma_k^{2} < \infty.
\end{align}
\section{Theoretical Analysis}

To analyze the asymptotic tracking capability of the projected stochastic gradient ascent updates, we present our auxiliary framework divided into three key structural steps: analyzing objective landscape geometry, bounding empirical gradient moments, and establishing continuous tracking approximations.

\subsection{Objective Function Geometry \& Landscape}
We first map the geometric landscape and invariant contours of the unsupervised energy objective $J_d(q)$ over the sphere $S^{d-1}$. The first step ensures our population objective is smooth and behaves predictably across the manifold.

\begin{lemma}[Smoothness of the Population Objective] \label{lem:lem1}
The unsupervised population function satisfies $J_{d}(q) \in C^{\infty}(S^{d-1})$.
\end{lemma}

\noindent Due to the isotropic profile of the random noise vectors, the high-dimensional objective collapses into a compact system determined solely by the latent axis projection.

\begin{lemma}[Rotational Invariance and Dimension Reduction] \label{lem:lem2}
There exists a smooth function $\Psi_{d}:[-1,1] \rightarrow \mathbb{R}$ such that $J_{d}(q) = \Psi_{d}(\langle q, \xi_{d} \rangle)$.
\end{lemma}

\noindent To measure the concentration of individual token weights under this representation, we provide a uniform matrix norm envelope bound on the softmax allocations.

\begin{lemma}[Concentration of Attention Weights] \label{lem:lem3}
The squared Euclidean norm of the attention weights satisfies 
\begin{equation}
    \|a(q)\|^2 =O\left( \frac{L e^{\beta^2} }{R(L-R)}\right) \qquad\text{a.s.}
\end{equation}
\end{lemma}

\noindent We complete the landscape characterization by establishing that the reduced representation is strictly monotonic along the alignment vector field, preventing spurious local traps.

\begin{lemma}[Monotonicity of the Reduced Objective] \label{lem:lem4}
As $R \to \infty, L-R \to \infty, \beta=O(1), \theta_d =O(1)$, the objective satisfies $\Psi_{d}^{\prime}(\rho) > 0$ for all intermediate alignment values $\rho \in [-1,1]$.
\end{lemma}

\subsection{Algorithmic Gradients \& Tracking Dynamics}
With the underlying landscape characteristics defined, we next isolate individual stochastic updates and use multi-index Taylor expansions to transition the discrete steps into continuous paths.

\begin{lemma}[Uniform Gradient Variance Bound] \label{lem:lem5}
Under the condition that the inverse temperature and signal strength parameters satisfy $\beta =O(1), \theta_d = O(1)$, the unsupervised stochastic gradient satisfies a dimension-dependent second moment condition:
\begin{equation}
    \mathbb{E}\left[ \|g_k\|^{2} \right] \le C d^{4}.
\end{equation}
\end{lemma}

\noindent By projecting the empirical gradient variance onto the tangent plane of the hypersphere, we decompose the discrete recurrence update into a true drift path and an optimization error envelope.

\begin{lemma}[Discrete Taylor Trajectory Projections] \label{lem:lem6}
The projected SGD updating equations expand as:
\begin{equation}
    q_{k+1} = q_k + \eta_k(I_d - q_k q_k^{\top})g_k + r_k,
\end{equation}
where the cumulative remainder vector satisfies the uniform bound $\|r_k\| \le C \eta_k^{2} \|g_k\|^{2}$.
\end{lemma}

\subsection{Stochastic Convergence \& Martingale Bounds}
The final phase of our auxiliary analysis provides the requisite martingale limits and Lipschitz tracking metrics required to execute the continuous ODE method.

\begin{lemma}[Global Drift Field Lipschitz Continuity] \label{lem:lem7}
The continuous drift vector field $H_{d}(q) = (I_d - q q^\top)\nabla_q J_d(q)$ is globally Lipschitz continuous on the hypersphere $S^{d-1}$.
\end{lemma}

\noindent To guarantee that empirical noise fluctuations balance out uniformly as iterations accumulate, we establish the convergence properties of the noise filtration array.

\begin{lemma}[Martingale Array Convergence Bounds] \label{lem:lem8}
The stochastically accumulating martingale sum series $\sum_{k=0}^{\infty} \eta_k  M_{k+1}$ converges almost surely.
\end{lemma}

\noindent Finally, we show that the accumulated mathematical approximations introduced during the Taylor projections vanish over time.

\begin{lemma}[Remainder Convergent Summability] \label{lem:lem9}
The cumulative Taylor remainder sequence is absolutely summable almost surely: 
\begin{equation}
    \sum_{k=0}^{\infty} \|r_k\| < \infty.
\end{equation}
\end{lemma}
\subsection{Time-Scale Mapping and Continuous Interpolation}
To study the discrete-time sequence $\{q_k\}_{k=0}^{\infty}$ using dynamical systems tools, we map the discrete iteration counter $k$ onto a continuous timeline. We define the cumulative time intervals $\tau_0 = 0$ and $\tau_k = \sum_{l=0}^{k-1} \eta_l$ for $k \ge 1$, where $\eta_l$ is the step size. 

Let $m(t) = \max\{k \ge 0 : \tau_k \le t\}$ denote the discrete step counter corresponding to the continuous time parameter $t$. We then construct the continuous-time piecewise constant interpolation trajectory $\bar{q}(t)$ defined by:
\begin{equation}
    \bar{q}(t) = q_{m(t)}, \quad t \in [0, \infty).
\end{equation}
Throughout the subsequent analysis, discrete-time iteration updates are denoted by a subscript index ($q_k$), whereas continuous-time trajectory coordinates and solutions to the limit differential equation $\dot{q} = H_d(q)$ are designated as functions of time ($q(t)$).
\subsection{Asymptotic Trajectory Tracking Proofs}

\begin{proposition}[Asymptotic Pseudo-Trajectory Approximation] \label{prop:tracking}
Let $q(t)$ be the continuous interpolation step trajectory. Under Lemmas~\ref{lem:lem1} through \ref{lem:lem9}, the discrete updates tracking sequence maps uniformly onto the continuous ODE system flow framework.
\end{proposition}

\begin{proof}[Proof of Proposition \ref{prop:tracking}]
Recall the one-step recursion established in Lemma~\ref{lem:lem6}:
\begin{align}
q_{k+1}= q_k + \eta_k H_d(q_k) + \eta_k M_{k+1} + r_k \label{eq16}.
\end{align}
To demonstrate convergence under the continuous framework, we partition our verification across the following structural stages:

\begin{description}
    \item[Step 1: Integral representation.] Fix $t \ge 0$ and $s \ge 0$. Summing the discrete recursion matching indices from $k = m(t)$ to $m(t+s)-1$ in \eqref{eq16} yields:
    \begin{align}
  \bar{q}(t+s)-\bar{q}(t)&=  q_{m(t+s)} - q_{m(t)} \nonumber \\
  &= \sum_{k=m(t)}^{m(t+s)-1} \eta_k H_d(q_k) + \sum_{k=m(t)}^{m(t+s)-1} \eta_k M_{k+1} + \sum_{k=m(t)}^{m(t+s)-1} r_k \label{eq17}.
    \end{align}
    
    \item[Step 2: Control of the interpolation error.] Let $k = m(u)$ for a given continuous time coordinate index. Evaluating the step offset bounds $\tau_k \le u < \tau_{k+1}$, the piecewise trajectory error satisfies:
    \begin{align}
    \| \bar{q}(u) - q_k \| = \| q_k - q_k \| = 0.
    \end{align}
    Consequently, we can write the main continuous tracking drift component directly as an integrated representation over the piecewise field:
    \begin{align}
    \sum_{k=m(t)}^{m(t+s)-1} \eta_k H_d(q_k) = \int_{t}^{t+s} H_d(\bar{q}(u)) \, du.
    \end{align}

    \item[Step 3: Evaluation of the continuous drift discrepancy.] To handle the difference between the actual trajectory path $q(t+u)$ and its step-wise representation $\bar{q}(t+u)$, we subtract the continuous differential equation component:
    \begin{align}
    q(t+s) - q(t) = \int_{t}^{t+s} H_d(q(u)) \, du \label{eq20}.
    \end{align}
    Taking the difference between the discrete interpolation series in~\eqref{eq17} and the continuous flow in~\eqref{eq20}  yields:
    \begin{align}
    \bar{q}(t+s) - q(t+s) = \bar{q}(t) - q(t) + \int_{t}^{t+s} \left[ H_d(\bar{q}(u)) - H_d(q(u)) \right] \, du + E(t,s)
    \end{align}
    where the combined error envelope tracking array maps directly to:
    \begin{align}
    E(t,s) = \sum_{k=m(t)}^{m(t+s)-1} \eta_k M_{k+1} + \sum_{k=m(t)}^{m(t+s)-1} r_k.
    \end{align}

    \item[Step 4: Application of Martingale and Remainder limits.] By Lemma~\ref{lem:lem8}, the martingale tracking array sequence $\sum_{k=0}^{\infty} \eta_k M_{k+1}$ converges almost surely. By Cauchy uniformity \cite{Royden}, this implies:
    \begin{align}
    \lim_{t \to \infty} \sup_{0 \le s \le T} \left\| \sum_{k=m(t)}^{m(t+s)-1} \eta_k M_{k+1} \right\| = 0 \qquad \text{a.s.}
    \end{align}
    Similarly, applying the absolute summability property from Lemma~\ref{lem:lem9}, the trailing remainder norm satisfies:
    \begin{align}
    \lim_{t \to \infty} \sup_{0 \le s \le T} \left\| \sum_{k=m(t)}^{m(t+s)-1} r_k \right\| \le \lim_{t \to \infty} \sum_{k=m(t)}^{\infty} \|r_k\| = 0 \qquad \text{a.s.}
    \end{align}
    Combining these limits under the supremum metric over a compact tracking window gives $\lim_{t \to \infty} \sup_{0 \le s \le T} \|E(t,s)\| = 0$ almost surely.

    \item[Step 5: Invoking the Lipschitz bound via Grönwall's Inequality.] Let $e(t) = \sup_{0 \le s \le T} \|\bar{q}(t+s) - q(t+s)\|$. Using the global Lipschitz constant $L$ proven in Lemma~\ref{lem:lem7}:
    \begin{align}
    \|\bar{q}(t+s) - q(t+s)\| \le \|\bar{q}(t) - q(t)\| + L \int_{t}^{t+s} \|\bar{q}(u) - q(u)\| \, du + \|E(t,s)\|.
    \end{align}
    Applying Grönwall's inequality~\cite{hirsch2013differential} yields:
    \begin{align}
    \sup_{0 \le s \le T} \|\bar{q}(t+s) - q(t+s)\| \le \left( \|\bar{q}(t) - q(t)\| + \sup_{0 \le s \le T} \|E(t,s)\| \right) e^{LT}.
    \end{align}

    \item[Step 6: Final asymptotic trajectory matching proof.] Taking the limit inferior and superior over time, the initialization offset $\|\bar{q}(t) - q(t)\|$ goes to 0 by definition of the matched start coordinates. Because $\lim_{t \to \infty} \sup_{0 \le s \le T} \|E(t,s)\| = 0$, we have:
    \begin{align}
    \lim_{t \to \infty} \sup_{0 \le s \le T} \|\bar{q}(t+s) - q(t+s)\| = 0 \qquad \text{a.s.}
    \end{align}
\end{description}
This completes the verification that the discrete-time optimization processes match the continuous-time ordinary differential equation limits almost surely.
\end{proof}

\section{Main Convergence Results}

With the static geometry of the landscape and the asymptotic tracking capabilities established in Section 5, we are now positioned to state our main convergence guarantees. We begin by formalizing the definition of the limit set of the continuous-time dynamical system, which acts as the target for our discrete stochastic trajectory.

\begin{definition}[Asymptotic Equilibrium Set] \label{def:equilibrium_set}
The equilibrium target set $\Lambda \subset S^{d-1}$ for the continuous projected field flow $H_d(q) = (I_d - qq^\top)\nabla_q J_d(q)$ is defined as the set of vanishing drift configurations:
\begin{equation}
    \Lambda = \left\{ q \in S^{d-1} : H_d(q) = 0 \right\}.
\end{equation}
\end{definition}

\noindent By leveraging the monotonicity property established in Lemma~\ref{lem:lem4}, we prove that this equilibrium set consists exclusively of the true latent signal direction and its exact inverse, showing that no spurious local maxima or saddle points can attract the optimization path inside the sphere.

\begin{theorem}[Global Convergence to the Signal Subspace] \label{thm:main}
Let $\{q_k\}_{k=0}^{\infty}$ be the sequence of query vectors generated by the projected stochastic gradient ascent update scheme on the hypersphere $S^{d-1}$. Assume the step-size sequence satisfies the standard Robbins-Monro conditions:
\begin{equation}
    \sum_{k=0}^{\infty} \eta_k = \infty \quad \text{and} \quad \sum_{k=0}^{\infty} \eta_k^2 < \infty.
\end{equation}
Furthermore, assume the system operates in the high-dimensional token scaling regime where $R \to \infty$, $L-R \to \infty$, and the inverse temperature and signal strength parameters satisfy $\beta =O(1), \theta_d = O(1)$. Then, the sequence of learned queries converges almost surely to the isolated latent signal subspace:
\begin{equation}
    \lim_{k \to \infty} \inf_{q^* \in \{\pm \xi_d\}} \|q_k - q^*\| = 0, \qquad \text{a.s.}
\end{equation}
\end{theorem}

\begin{proof}[Proof of Theorem~\ref{thm:main}]
Let
\begin{align*}
H_d(q)
=
(I_d - qq^\top)\nabla J_d(q)
\end{align*}
denote the limiting drift field, and consider the ODE
\begin{align}
\dot q = H_d(q).
\end{align}

By Lemma~\ref{lem:lem1}, $H_d$ is continuous on the compact manifold $S^{d-1}$.

We first characterize the equilibrium set of the ODE. Along any solution $q(t)$,
\begin{align*}
\frac{d}{dt}J_d(q(t))
&=
\left\langle \nabla J_d(q(t)), \dot q(t) \right\rangle \\
&=
\left\langle \nabla J_d(q(t)), (I_d - q(t)q(t)^\top)\nabla J_d(q(t)) \right\rangle \\
&=
\|(I_d - q(t)q(t)^\top)\nabla J_d(q(t))\|^2 \\
&=
\|H_d(q(t))\|^2 \ge 0.
\end{align*}

Moreover,
\begin{align}
\frac{d}{dt}J_d(q(t)) = 0
\quad \text{if and only if} \quad
H_d(q(t)) = 0.
\end{align}
Hence, $J_d$ is a strict Lyapunov function for the flow.

Now, by Lemma~\ref{lem:lem2} and Lemma~\ref{lem:lem4}, we have
\begin{align}
\nabla J_d(q)
=
\Psi_d'(\rho)\,\xi_d,
\qquad
\rho=\langle q,\xi_d\rangle,
\end{align}
and
\begin{align}
\Psi_d'(\rho) > 0,
\qquad
\rho \in [-1,1].
\end{align}

Hence, the equilibrium condition becomes
\begin{align}
\Psi_d'(\rho)\,(I_d - qq^\top)\xi_d = 0,
\end{align}
which implies
\begin{align}
(I_d - qq^\top)\xi_d = 0.
\end{align}

Therefore,
\begin{align}
\xi_d
=
qq^\top\xi_d
=
\langle q,\xi_d\rangle q
=
\rho q.
\end{align}

Taking norms gives
\begin{align}
1 = \|\xi_d\|
=
|\rho|\,\|q\|
=
|\rho|,
\end{align}
and therefore
\begin{align}
\rho = \pm 1.
\end{align}

Consequently, the equilibrium set is
\begin{align}
\Lambda = \{\xi_d, -\xi_d\}.
\end{align}

Since $J_d$ is a strict Lyapunov function, every internally chain transitive invariant set of the flow is contained in $\Lambda$~\cite{benaim1999dynamics}. 

Now, by Proposition~\ref{prop:tracking}, the affine interpolation of the stochastic iterates forms an asymptotic pseudo-trajectory of the ODE. Hence, by the Kushner–Clark theorem \cite{kushner1978, benaim1999dynamics}, the limit set $\omega(\{q_k\})$ is almost surely an internally chain transitive invariant set of the flow generated by $H_d$.

Since the only internally chain transitive invariant sets are the equilibria $\pm \xi_d$, it follows that
\begin{align}
\omega(\{q_k\})
\subseteq
\{\pm\xi_d\}
\qquad \text{a.s.}
\end{align}

Therefore,
\begin{align}
\lim_{k\to\infty}
\inf_{q^\ast\in\{\pm\xi_d\}}
\|q_k-q^\ast\|
=
0
\qquad \text{a.s.}
\end{align}

This completes the proof.
\end{proof}

\noindent As an immediate structural consequence of Theorem~\ref{thm:main}, we can guarantee that the empirical softmax attention weights concentration profile converges to a clean, deterministic indicator mapping. This confirms that the model asymptotically ignores the high-dimensional background noise.

\begin{corollary}[Asymptotic Attention Partition Filtering] \label{corr:attention_partition}
 Under the global convergence guarantees of Theorem~\ref{thm:main}, if the parameter scalings satisfy $\beta=O(1)$, $ \theta_d = O(1)$ and the relative token distribution obeys $\frac{R}{L-R} \to \infty$ as $R \to \infty$ and $L-R \to \infty$, the total attention weight allocated to the background noise partition vanishes almost surely:
\begin{equation}
    \lim_{k \to \infty} \sum_{j \in \mathcal{I}} a_j(q_k, X) = 0 \qquad \text{a.s.}
\end{equation}
Consequently, the entire attention allocation mass concentrates within the informative token partition:
\begin{equation}
    \lim_{k \to \infty} \sum_{i \in \mathcal{N}_o} a_i(q_k,X) = 1 \qquad \text{a.s.}
\end{equation}
\end{corollary}
\begin{remark}[Separation of Convergence and Attention Concentration]
We highlight a fundamental structural distinction between the requirements for geometric query convergence (Theorem~\ref{thm:main}) and absolute partition filtering (Corollary~\ref{corr:attention_partition}). The optimization process requires only $R \to \infty$ and $L-R \to \infty$ because isotropic noise cancels out in expectation within the gradient drift field, allowing $q_k$ to successfully track the signal direction. However, at the equilibrium point, the static softmax distribution requires the stronger relative density condition $\frac{R}{L-R} \to \infty$ to collectively drown out the noise partition when individual signal components are weak ($\beta\theta_d = O(1)$).
\end{remark}
\begin{proof}
Recall that under the inverse temperature parameter $\beta > 0$, the attention weight $a_m(q_k,X)$ assigned to an arbitrary token $x_m$ at step $k$ is given by the softmax function:
\begin{align}
    a_m(q_k,X) = \frac{\exp(\beta \langle q_k, x_m \rangle)}{\sum_{l \in \mathcal{I}} \exp(\beta \langle q_k, x_l \rangle) + \sum_{n \in \mathcal{N}_o} \exp(\beta \langle q_k, x_n \rangle)}.
\end{align}
By Theorem~\ref{thm:main}, the learned query sequence converges almost surely to the isolated signal subspace, meaning $\lim_{k \to \infty} q_k = q^* \in \{\pm \xi_d\}$ a.s. Without loss of generality, assume $q^* = \xi_d$.

We now apply the continuous mapping theorem to analyze the total attention mass allocated to the background noise partition, $\sum_{j \in \mathcal{N}_o} a_j(q_k,X)$, as $t \to \infty$:
\begin{align}
    \lim_{k \to \infty} \sum_{j \in \mathcal{N}_o} a_j(q_k,X) = \frac{\sum_{j \in \mathcal{N}_o} \exp\left(\beta \langle \xi_d, z_j \rangle\right)}{\sum_{l \in \mathcal{I}} \exp\left(\beta \theta_d + \beta \langle \xi_d, z_l \rangle\right) + \sum_{n \in \mathcal{N}_o} \exp\left(\beta \langle \xi_d, z_n \rangle\right)}.
\end{align}
To evaluate this fraction under the partitioned scaling regime, we divide both the numerator and the denominator by the total sum over the nuisance elements, $\sum_{n \in \mathcal{N}_o} \exp\left(\beta \langle \xi_d, z_n \rangle\right)$. This simplifies the total mass expression to:
\begin{align}
    \lim_{k \to \infty} \sum_{j \in \mathcal{N}_o} a_j(q_k,X) = \frac{1}{\exp(\beta \theta_d) \cdot \frac{\sum_{l \in \mathcal{I}} \exp\left(\beta \langle \xi_d, z_l \rangle\right)}{\sum_{n \in \mathcal{N}_o} \exp\left(\beta \langle \xi_d, z_n \rangle\right)} + 1} \label{eq30}.
\end{align}
Because $z_i \sim \mathcal{N}(0, I_d)$ are independent and identically distributed, their projections along the fixed signal vector $\xi_d$ are i.i.d. scalar Gaussians $\langle \xi_d, z_i \rangle \sim \mathcal{N}(0, 1)$. By the Kolmogorov Law of Large Numbers, as the partition sets grow large ($R \to \infty$ and $L-R \to \infty$), the empirical means converge to their identical expected values $\mathbb{E}[e^{\beta \langle \xi_d, z \rangle}] = e^{\beta^2/2}$ almost surely:
\begin{align}
    \frac{1}{R}\sum_{l \in \mathcal{I}} \exp\left(\beta \langle \xi_d, z_l \rangle\right) &\xrightarrow{\text{a.s.}} e^{\beta^2/2}, \\
    \frac{1}{L-R}\sum_{n \in \mathcal{N}_o} \exp\left(\beta \langle \xi_d, z_n \rangle\right) &\xrightarrow{\text{a.s.}} e^{\beta^2/2}.
\end{align}
Evaluating the ratio of these two sums under their respective dimensions gives:
\begin{align}
    \frac{\sum_{l \in \mathcal{I}} \exp\left(\beta \langle \xi_d, z_l \rangle\right)}{\sum_{n \in \mathcal{N}_o} \exp\left(\beta \langle \xi_d, z_n \rangle\right)} = \frac{R \cdot \left(e^{\beta^2/2} + o(1)\right)}{(L-R) \cdot \left(e^{\beta^2/2} + o(1)\right)} = \frac{R}{L-R}(1 + o(1)) \qquad \text{a.s.}  \label{eq31}
\end{align}
By the high-dimensional token scaling and concentration criteria defined in Lemma~\ref{lem:lem3}, the relative density parameter ratio $\frac{R}{L-R}$ grows unboundedly large as the systems expand ($\frac{R}{L-R} \to \infty$). Substituting the asymptotic behavior in \eqref{eq31} back into the total mass fraction in \eqref{eq30} yields:
\begin{align}
    \lim_{k \to \infty} \sum_{j \in \mathcal{N}_o} a_j(q_k,X) = \frac{1}{\exp(\beta \theta_d) \cdot \infty + 1} = 0 \qquad \text{a.s.}
\end{align}
Consequently, because the sum of all attention weights must equal 1, the total mass of the informative partition matches:
\begin{align}
    \lim_{k \to \infty} \sum_{i \in \mathcal{I}} a_i(q_k,X) = 1 - \lim_{k \to \infty} \sum_{j \in \mathcal{N}_o} a_j(q_k,X) = 1 - 0 = 1 \qquad \text{a.s.}
\end{align}
This completes the proof.

\end{proof}
\section{Experiments}
\begin{figure}[htbp]
    \centering
    \includegraphics[width=0.8\textwidth]{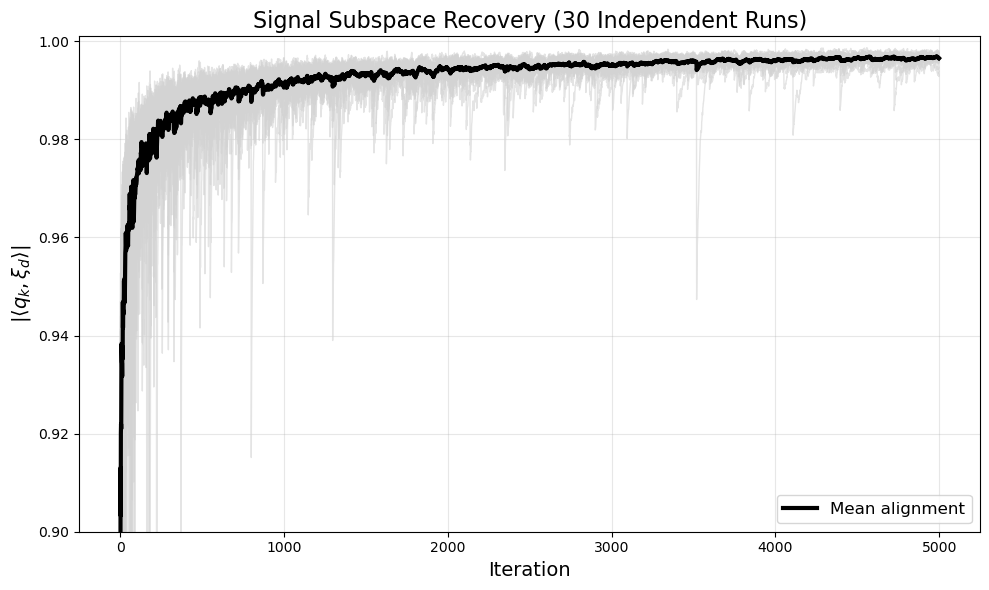}
    \caption{
Alignment $|\langle q_k,\xi_d\rangle|$ between the learned query and
the latent signal direction over $30$ independent runs. Gray curves
correspond to individual trajectories, while the black curve denotes
their average.
}
    \label{fig:alignment}
\end{figure}
Figure~\ref{fig:alignment} illustrates the evolution of the alignment $|\langle q_k,\xi_d\rangle|$ during the optimization process. The alignment increases rapidly and approaches one in all runs, providing empirical evidence that the attention dynamics recover the latent signal direction. Across 30 independent trials, the final alignment has mean $0.9965$ and standard deviation $0.0011$, with minimum $0.9931$ and maximum $0.9980$. These results demonstrate highly stable signal recovery and are consistent with the asymptotic convergence predicted by Theorem~\ref{thm:main}.

\begin{figure}[htbp]
    \centering
    \includegraphics[width=0.8\textwidth]{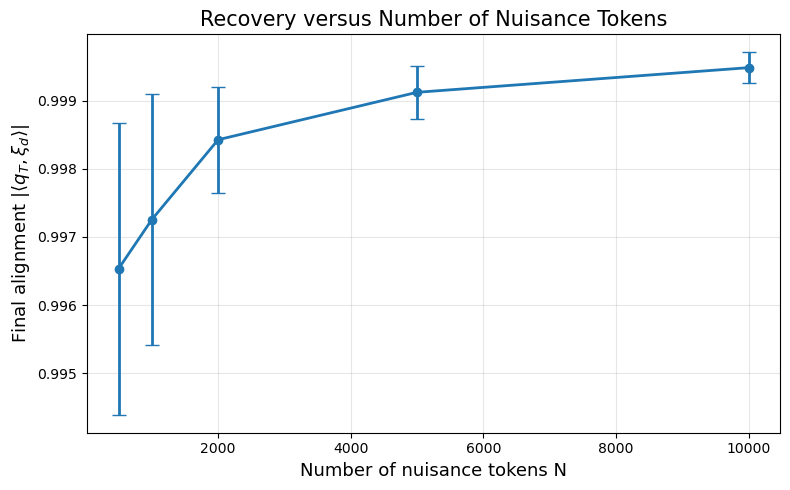}
    \caption{
Final alignment $|\langle q_T,\xi_d\rangle|$ as a function of the number of nuisance tokens $N$, with the number of informative tokens fixed at $R=500$. Error bars represent two standard deviations over 30 independent runs. Recovery remains robust even as the number of nuisance tokens increases, supporting the theoretical prediction that the attention dynamics can identify the latent signal direction in the presence of many irrelevant tokens.
}
    \label{fig:nuisance}
\end{figure}

To evaluate the robustness of the proposed dynamics with respect to nuisance tokens, we fix $R=500$ informative tokens and vary the number of nuisance tokens from $N=500$ to $N=10000$. For each setting, the projected stochastic gradient algorithm is run for $T=5000$ iterations and repeated over $30$ independent trials. Figure~\ref{fig:nuisance} shows the resulting final alignment $|\langle q_T,\xi_d\rangle|$. The final alignment remains very close to one throughout the entire range of nuisance-token levels considered. In particular, increasing the number of nuisance tokens by a factor of twenty does not significantly degrade recovery performance. This observation provides empirical evidence that the attention dynamics are capable of isolating and amplifying informative structure even in highly cluttered environments, supporting the asymptotic recovery guarantees established in Theorem~\ref{thm:main}.

\begin{figure}[htbp]
    \centering
    \includegraphics[width=0.8\textwidth]{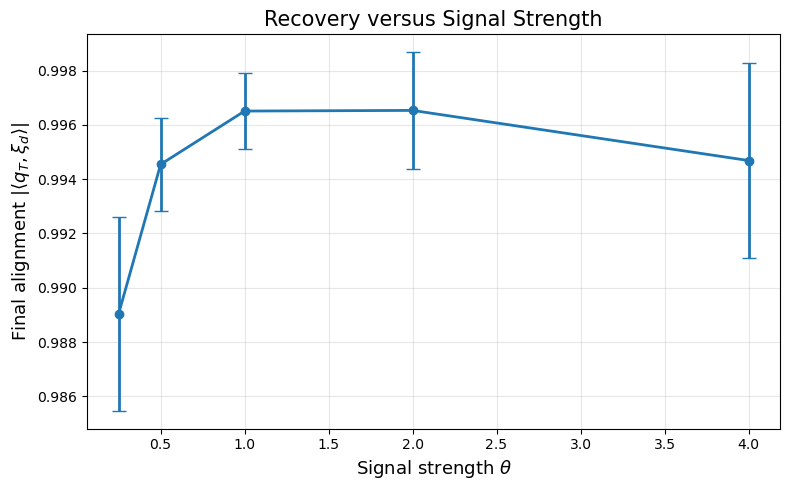}
    \caption{
Final alignment $|\langle q_T,\xi_d\rangle|$ as a function of the signal strength $\theta$. The numbers of informative and nuisance tokens are fixed at $R=N=500$, and the projected stochastic gradient algorithm is run for $T=5000$ iterations. Error bars represent two standard deviations over $30$ independent trials. The final alignment remains close to one across all tested values of $\theta$, indicating that signal subspace recovery is robust even in the weak-signal regime.
}
    \label{fig:signalstrength}
\end{figure}
To examine the influence of signal strength, we fix the numbers of informative and nuisance tokens at $R=N=500$ and vary the signal amplitude according to
$\theta \in \{0.25,0.5,1,2,4\}$. For each value of $\theta$, the projected stochastic gradient algorithm is run for $T=5000$ iterations and repeated over $30$ independent trials. Figure~\ref{fig:signalstrength} reports the resulting final alignment $|\langle q_T,\xi_d\rangle|$. The recovery performance remains remarkably stable across all tested signal strengths. Even for the weakest signal level $\theta=0.25$, the mean final alignment exceeds $0.98$, while for all values of $\theta$ the mean alignment remains above $0.99$. These observations indicate that successful recovery does not rely on an increasingly strong signal and instead arises from the ability of the attention mechanism to aggregate information across many informative tokens. The experimental findings are therefore consistent with the theoretical prediction that the attention dynamics recover the latent signal subspace without requiring a diverging signal-to-noise ratio.

\clearpage

\section{Appendices}
\subsection{Proof of Lemma \ref{lem:lem1}}
Let $q$ be an arbitrary vector in the embedding domain. Consider the joint functional mapping:
\begin{align}
(q, X) \mapsto a_i(q, X) = \frac{e^{\beta \langle x_i, q \rangle}}{\sum_{j=1}^L e^{\beta \langle x_j, q \rangle}} \label{eq57b}.
\end{align}
This mapping is a composition of linear products and exponential terms. Since the denominator is a sum of positive exponentials, it is strictly positive for all bounded realizations of $X$ and $q$, preventing singularities. Thus, the mapping is smooth jointly in $q$ and $X$.

We calculate the exact partial derivative with respect to a query coordinate $q_k$. From \eqref{eq57b} we have:
\begin{align}
\frac{\partial a_i(q, X)}{\partial q_k} &= \frac{\beta x_{ik} e^{\beta \langle x_i, q \rangle} \sum_{j=1}^L e^{\beta \langle x_j, q \rangle} - e^{\beta \langle x_i, q \rangle} \sum_{j=1}^L \beta x_{jk} e^{\beta \langle x_j, q \rangle}}{\left(\sum_{j=1}^L e^{\beta \langle x_j, q \rangle}\right)^2} \nonumber\\
&= \beta x_{ik} a_i(q, X) - \beta a_i(q, X) \sum_{j=1}^L a_j(q, X) x_{jk} \nonumber\\
&= \beta a_i(q, X) \left( x_{ik} - f_q(X)_k \right).
\end{align}
By induction, any arbitrary multi-index derivative of order $\alpha$, denoted $\partial_q^\alpha a_i(q, X)$, can be expressed as a finite sum of algebraic products of polynomials in the token coordinates and products of attention weights:
\begin{align}
\partial_q^\alpha a_i(q, X) = \sum_{m} P_m(x_1, \dots, x_L) \prod_{j \in \mathcal{S}_m} a_j(q, X) \label{eq55},
\end{align}
where each $P_m$ is a multivariate polynomial mapping. Because the softmax weights are bounded within the unit interval ($0 < a_j < 1$), from \eqref{eq55} we obtain:
\begin{align}
\left| \partial_q^\alpha a_i(q, X) \right| \le \sum_m |P_m(x_1, \dots, x_L)| \le C_\alpha \left(1 + \max_{1 \le j \le L} \|x_j\|^{m_\alpha}\right).
\end{align}
where $m_\alpha \in \mathbb{N}$ depends on the order $|\alpha|$.

Now, consider the pooled attention output function $f_q(X) = \sum_{i=1}^L a_i(q, X) x_i$. Its derivative satisfies:
\begin{align}
\left\| \partial_q^\alpha f_q(X) \right\| \le \sum_{i=1}^L \left| \partial_q^\alpha a_i(q, X) \right| \cdot \|x_i\| \le C_\alpha' \left(1 + \max_{1 \le j \le L} \|x_j\|^{m_\alpha + 1}\right).
\end{align}
The unsupervised population objective function is $J_d(q) = \mathbb{E}\left[ \|f_q(X)\|^2 \right] = \mathbb{E}\left[ \langle f_q(X), f_q(X) \rangle \right]$. Applying the product rule, the partial derivatives of the integrand are bounded by the polynomial envelope of the token inputs:
\begin{align}
\left| \partial_q^\alpha \left( \langle f_q(X), f_q(X) \rangle \right) \right| \le C_\alpha'' \left(1 + \max_{1 \le j \le L} \|x_j\|^{2m_\alpha + 2}\right).
\end{align}
To differentiate under the expectation operator, we must verify that the polynomial envelope is integrable. Because each token $x_i = v_i \theta_d \xi_d + z_i$ is a shifted standard Gaussian vector, its norm components possess finite moments of all orders. Specifically, using Gaussian tail properties, the maximum norm across $L$ variables satisfies:
\begin{align}
\mathbb{E}\left[ \max_{1 \le j \le L} \|x_j\|^k \right] \le 2^k \left( \theta_d^k + \mathbb{E}\left[ \max_{1 \le j \le L} \|z_j\|^k \right] \right) < \infty, \quad \forall k \in \mathbb{N}.
\end{align}
Since the integrand derivatives are bounded by an integrable function independent of $q$, the Lebesgue Dominated Convergence Theorem permits the exchange of differentiation and expectation to all orders \cite{Royden}. Hence, $J_d(q) \in C^\infty(S^{d-1})$.

\subsection{Proof of Lemma \ref{lem:lem2}}
Let $U \in O(d) = \{M \in \mathbb{R}^{d \times d} : M^\top M = I_d\}$ be an arbitrary orthogonal transformation matrix that leaves the hidden signal vector invariant:
\begin{align}
U \xi_d = \xi_d.
\end{align}
We evaluate the statistical distribution of the rotated input token matrix $UX = [Ux_1, Ux_2, \dots, Ux_L]^\top$. Applying the linear transformation to each token model equation:
\begin{align}
Ux_i = v_i \theta_d U \xi_d + Uz_i = v_i \theta_d \xi_d + Uz_i.
\end{align}
Because the noise component $z_i \sim \mathcal{N}(0, I_d)$ is an isotropic multivariate Gaussian vector, its probability density function is invariant under orthogonal rotation. Thus, $Uz_i \sim \mathcal{N}(0, U I_d U^\top) = \mathcal{N}(0, I_d)$. Since the indicators $v_i$ and the signal direction $\xi_d$ are unchanged, the joint distribution of the transformed token matrix matches the original matrix:
\begin{align}
(Ux_1, Ux_2, \dots, Ux_L) \overset{d}{=} (x_1, x_2, \dots, x_L).
\end{align}
Now, we evaluate the attention weights under a rotated query $Uq$ and rotated data inputs $UX$:
\begin{align}
a_i(Uq, UX) = \frac{e^{ \beta \langle Ux_i, Uq \rangle }}{\sum_{j=1}^L e^{\beta \langle Ux_j, Uq \rangle }}. \label{eq63} 
\end{align}
Since $U$ is orthogonal, $U^\top U = I_d$, which implies $\langle Ux_i, Uq \rangle = \langle x_i, q \rangle$. Substituting this back to \eqref{eq63} yields:
\begin{align}
a_i(Uq, UX) = \frac{e^{\beta \langle x_i, q \rangle}}{\sum_{j=1}^L e^{\beta \langle x_j, q \rangle}} = a_i(q, X) \label{eq68b}.
\end{align}
Using the identity \eqref{eq68b}, we express the pooled output function under rotation:
\begin{align}
f_{Uq}(UX) = \sum_{i=1}^L a_i(Uq, UX) (Ux_i) = \sum_{i=1}^L a_i(q, X) U x_i = U \left( \sum_{i=1}^L a_i(q, X) x_i \right) = U f_q(X) \label{eq65}.
\end{align}
Now, we evaluate the unsupervised population objective function at the transformed query vector $Uq$:
\begin{align}
J_d(Uq) = \mathbb{E}_{X}\left[ \|f_{Uq}(X)\|^2 \right].
\end{align}
Since the expectation is taken over the entire distribution of the token matrix, and the distribution is invariant under rotation ($X \overset{d}{=} UX$), we change variables to $UX$:
\begin{align}
J_d(Uq) = \mathbb{E}_{X}\left[ \|f_{Uq}(UX)\|^2 \right] = \mathbb{E}_{X}\left[ \|U f_q(X)\|^2 \right] \label{eq67},
\end{align} where the second equality in \eqref{eq67} follows from \eqref{eq65}. 

The vector norm is invariant under orthogonal matrix operations ($\|U v\|^2 = v^\top U^\top U v = \|v\|^2$). Therefore, from \eqref{eq67} we obtain:
\begin{align}
J_d(Uq) = \mathbb{E}_{X}\left[ \|f_q(X)\|^2 \right] = J_d(q).
\end{align}
This equality holds for any orthogonal transformation $U$ that fixes $\xi_d$. By the transitive action of the stabilizer subgroup
$
\mathrm{Stab}(\xi_d)
=
\{U\in O(d):U\xi_d=\xi_d\}
\cong O(d-1)
$
on the unit sphere of $\xi_d^\perp$, any two vectors
$q_1,q_2\in S^{d-1}$ satisfying
$
\langle q_1,\xi_d\rangle
=
\langle q_2,\xi_d\rangle
$
can be mapped to one another by an orthogonal transformation $U\in O(d)$ such that
$
U\xi_d=\xi_d,
Uq_1=q_2
$
(see, e.g., Lee~\cite{lee2013smooth}). Thus, $J_d(q)$ is constant on the contours where $\langle q, \xi_d \rangle$ is constant. This guarantees the existence of a unique smooth scalar mapping $\Psi_d: [-1,1] \rightarrow \mathbb{R}$ such that $J_d(q) = \Psi_d(\langle q, \xi_d \rangle)$.

\subsection{Proof of  Lemma \ref{lem:lem3}}
The squared Euclidean norm of the full attention weight vector is the sum of the squared weights across both partitions:
\begin{align}
\|a(q)\|^2 = \sum_{i=1}^L a_i(q, X)^2 = \sum_{i \in \mathcal{I}} a_i(q, X)^2 + \sum_{j \in \mathcal{N}_o} a_j(q, X)^2 \label{P0}.
\end{align}
Let $\rho = \langle q, \xi_d \rangle$. We decompose each background noise vector $z_i$ into a component aligned with the query vector $q$ and an orthogonal component residing in the subspace $q^\perp$:
\begin{align}
z_i = G_i q + z_i^\perp, \quad \text{where } G_i = \langle z_i, q \rangle \sim \mathcal{N}(0,1), \quad z_i^\perp \sim \mathcal{N}(0, I_d - q q^\top).
\end{align}
We evaluate the exact inner product projections $\langle x_i, q \rangle$ for the tokens based on their partition identity. For informative tokens ($i \in \mathcal{I}$):
\begin{align}
\langle x_i, q \rangle = v_i \theta_d \langle \xi_d, q \rangle + \langle z_i, q \rangle = \theta_d \rho + G_i, \quad G_i \overset{i.i.d.}{\sim} \mathcal{N}(0,1).
\end{align}
For noise tokens ($j \in \mathcal{N}_o$):
\begin{align}
\langle x_j, q \rangle = 0 \cdot \theta_d \langle \xi_d, q \rangle + \langle z_j, q \rangle = G_j', \quad G_j' \overset{i.i.d.}{\sim} \mathcal{N}(0,1).
\end{align}
 We isolate the deterministic signal scaling factors from the random components in the informative partition sum $A_{\rho}$:
\begin{align}
A_{\rho} = \sum_{i \in \mathcal{I}} e^{\beta (\theta_d \rho + G_i)} = e^{\beta \theta_d \rho} \sum_{i \in \mathcal{I}} e^{\beta G_i} \label{eq73},
\end{align}
where $G_i \overset{i.i.d.}{\sim} \mathcal{N}(0,1)$. The terms $e^{\beta G_i}$ are independent and identically distributed log-normal ran
dom variables. Their theoretical expected value is given by the moment-generating function of a standard normal distribution:
\begin{align}
\mathbb{E}\left[ e^{\beta G_i} \right] = e^{\beta^2 / 2} \label{eq75}.
\end{align}
By the Strong Law of Large Numbers, the empirical sample average converges in probability to this expected value as the number of informative tokens grows ($R \rightarrow \infty$):
\begin{align}
\text{lim}_{R \rightarrow \infty} \frac{1}{R} \sum_{i \in \mathcal{I}} e^{\beta G_i} = e^{\beta^2 / 2} \qquad\text{a.s.} \label{eq77}.
\end{align}
From \eqref{eq73} and \eqref{eq77} we have 
\begin{align}
A_{\rho} = R \cdot e^{\beta \theta_d \rho + \beta^2 / 2} \left( 1 + o(1) \right) \qquad\text{a.s.}
\end{align}
We apply the same analysis to the noise partition sum $B = \sum_{j \in \mathcal{N}_o} e^{\beta G_j'}$. The terms $e^{\beta G_j'}$ are also i.i.d. log-normal variables with the same expected value: $\mathbb{E}[e^{\beta G_j'}] = e^{\beta^2 / 2}$. Since the number of noise tokens is $L-R$, the sample average satisfies:
\begin{align}
\text{lim}_{L-R \rightarrow \infty} \frac{1}{L-R} \sum_{j \in \mathcal{N}_o} e^{\beta G_j'} = e^{\beta^2 / 2} \qquad\text{a.s.} \implies B = (L-R) \cdot e^{\beta^2 / 2} \left( 1 + o(1) \right) \qquad\text{a.s.}
\end{align}
Now, observe that for each $i \in  \mathcal{I}$,
\begin{align}
0<a_i(q, X) = \frac{e^{\beta(\theta_d \rho + G_i)}}{A_{\rho} + B} \leq  \frac{e^{\beta(\theta_d \rho + G_i)}}{A_{\rho} },
\end{align}
which leads to 
\begin{align}
a_i(q,X)^2 \leq \frac{e^{2\beta(\theta_d \rho + G_i)}}{A_{\rho}^2 }.
\end{align}
Hence, it follows that
\begin{align}
\sum_{i \in \mathcal{I}} a_i(q, X)^2 &\leq \sum_{i \in \mathcal{I}}\frac{e^{2\beta(\theta_d \rho + G_i)}}{A_{\rho}^2 }\\
&= \sum_{i \in \mathcal{I}}\frac{e^{2\beta(\theta_d \rho + G_i)}}{(R \cdot e^{\beta \theta_d \rho + \beta^2 / 2} \left( 1 + o(1) \right))^2 }\\
&= \frac{1}{e^{\beta^2} R} \bigg[\frac{1}{R} \sum_{i \in \mathcal{I}}e^{2\beta G_i}\bigg](1+o(1)) \qquad\text{a.s.} \label{K1}. 
\end{align}
The terms in both sums are independent log-normal random variables. Applying the Strong Law of Large Numbers, these sample averages converge in probability to their theoretical expected values:
\begin{align}
\text{lim}_{R \rightarrow \infty} \frac{1}{R} \sum_{i \in \mathcal{I}} e^{2 \beta G_i} &= \mathbb{E}\left[ e^{2 \beta G_i} \right] = e^{(2\beta)^2 / 2} = e^{2 \beta^2}  \qquad\text{a.s.} \label{K2}.
\end{align}
From \eqref{K1} and \eqref{K2} we obtain
\begin{align}
\sum_{i \in \mathcal{I}} a_i(q, X)^2\leq  \frac{e^{\beta^2}}{R}(1+o(1)) \qquad\text{a.s.} \label{K3}.
\end{align}
Similarly, for each $j \in \mathcal{N}_0$ we have
\begin{align}
0< a_j(q, X) = \frac{e^{\beta G_j'}}{A + B} \leq  \frac{e^{\beta G_j'}}{B },
\end{align}
which leads to 
\begin{align}
a_j(q,X)^2 \leq \frac{e^{2\beta G_j'}}{B^2 }.
\end{align}
Hence, it follows that
\begin{align}
\sum_{j \in \mathcal{N}_0} a_j(q, X)^2 &\leq \sum_{i \in \mathcal{I}}\frac{e^{2\beta G_j'}}{B^2 }\\
&= \sum_{j \in \mathcal{N}_0}\frac{e^{2\beta G_j'}}{((L-R) \cdot e^{ \beta^2 / 2} \left( 1 + o(1) \right))^2 }\\
&= \frac{1}{e^{\beta^2} (L-R)} \bigg[\frac{1}{(L-R)} \sum_{j \in \mathcal{N}_0}e^{2\beta G_j'}\bigg](1+o(1)) \qquad\text{a.s.} \label{K5}. 
\end{align}
The terms in both sums are independent log-normal random variables. Applying the Weak Law of Large Numbers, these sample averages converge in probability to their theoretical expected values:
\begin{align}
\text{plim}_{L-R \rightarrow \infty} \frac{1}{L-R} \sum_{j \in \mathcal{N}_0} e^{2 \beta G_j'} &= \mathbb{E}\left[ e^{2 \beta G_j'} \right] = e^{(2\beta)^2 / 2} = e^{2 \beta^2}  \label{K6}.
\end{align}
From \eqref{K5} and \eqref{K6} we obtain
\begin{align}
\sum_{j \in \mathcal{N}_0} a_i(q, X)^2\leq  \frac{e^{\beta^2}}{L-R}(1+o(1)) \qquad\text{a.s.} \label{K3}.
\end{align}
We substitute these limits back into Equation \ref{P0}:
\begin{align}
\|a(q)\|^2 \leq \frac{e^{\beta^2}}{R}(1+o(1)) + \frac{e^{\beta^2}}{L-R}(1+o_P(1))\leq  \frac{Le^{\beta^2}}{R(L-R)} (1+o(1)) \qquad\text{a.s.}.
\end{align}
It follows that
\begin{align}
\|a(q)\|^2 = O\left( \frac{Le^{\beta^2}}{R(L-R)}\right).
\end{align}
This concludes our proof of Lemma \ref{lem:lem3}. 

\subsection{Proof of Lemma \ref{lem:lem4}}
Before proving Lemma \ref{lem:lem4}, we prove the following result. 
\begin{lemma} \label{lem:aux1} 
Let
\begin{align}
S_L(\rho,X)
=
\sum_{i=1}^{L} a_i^2(q,X),
\end{align}
where
\begin{align}
a_i(q,X)
=
\frac{\exp(\beta y_i)}
{\sum_{j=1}^{L}\exp(\beta y_j)},
\end{align}
and
\begin{align}
y_i
=
\theta_d v_i \rho + G_i,
\end{align}
with $G_1,\ldots,G_L$ i.i.d.\ standard Gaussian random variables.
Assume that
\begin{align}
S_L(\rho,X)\to 0
\qquad\text{a.s.}
\end{align}
for every $\rho\in[-1,1]$.
Then
\begin{align}
\frac{\partial}{\partial\rho}
\mathbb E[S_L(\rho,X)]
\longrightarrow 0.
\end{align}
\end{lemma}

\begin{proof}
Define
\begin{align}
F_L(\rho)
=
\mathbb E[S_L(\rho,X)].
\end{align}
Since
\begin{align}
y_i
=
\theta_d v_i\rho+G_i,
\end{align}
we may write
\begin{align}
F_L(\rho)
=
\int_{\mathbb R^L}
S_L(\theta_d v_1\rho+g_1,\ldots,
\theta_d v_L\rho+g_L)
\varphi(g)\,dg,
\end{align}
where
\begin{align}
\varphi(g)
=
(2\pi)^{-L/2}
\exp\!\left(
-\frac{\|g\|^2}{2}
\right)
\end{align}
denotes the standard Gaussian density on $\mathbb R^L$.

Since the softmax map is smooth, $S_L$ is continuously differentiable with respect to the coordinates $(y_1,\ldots,y_L)$. Therefore
\begin{align}
\frac{\partial S_L}{\partial y_k}
=
2\sum_{i=1}^{L}
a_i
\frac{\partial a_i}{\partial y_k}.
\end{align}
Using the softmax derivative formula
\begin{align}
\frac{\partial a_i}{\partial y_k}
=
\beta a_i
(\delta_{ik}-a_k),
\end{align}
we obtain
\begin{align}
\frac{\partial S_L}{\partial y_k}
&=
2\beta
\sum_{i=1}^{L}
a_i^2
(\delta_{ik}-a_k)
\\
&=
2\beta
\left(
a_k^2
-
a_k
\sum_{i=1}^{L}a_i^2
\right)
\\
&=
2\beta
\bigl(
a_k^2-a_kS_L
\bigr).
\end{align}

Since
\begin{align}
0\le a_k\le 1,
\qquad
0\le S_L\le 1,
\end{align}
it follows that
\begin{align}
\left|
\frac{\partial S_L}{\partial y_k}
\right|
&\le
2\beta
\left(
a_k^2+a_kS_L
\right)
\\
&\le
4\beta a_k.
\end{align}

Hence
\begin{align}
\left|
\frac{\partial}{\partial\rho}
S_L(\theta_d v\rho+g)
\right|
&=
\theta_d
\left|
\sum_{k=1}^{L}
v_k
\frac{\partial S_L}{\partial y_k}
\right|
\\
&\le
4\beta\theta_d
\sum_{k=1}^{L}
v_k a_k
\\
&\le
4\beta\theta_d.
\end{align}

Therefore the derivative is dominated by the integrable constant
\begin{align}
4\beta\theta_d.
\end{align}
The Dominated Convergence Theorem \cite{Royden} yields
\begin{align}
F_L'(\rho)
&=
\theta_d
\sum_{k=1}^{L}
v_k
\mathbb E
\left[
\frac{\partial S_L}{\partial y_k}
\right].
\end{align}

Substituting the expression for $\partial S_L/\partial y_k$,
\begin{align}
|F_L'(\rho)|
&\le
2\beta\theta_d
\,
\mathbb E
\left[
\sum_{k=1}^{L}
v_k
\left(
a_k^2+a_kS_L
\right)
\right]
\\
&= 2\beta\theta_d \mathbb E \left[ \sum_{k=1}^{L} v_k a_k^2\right]+  2\beta\theta_d \mathbb E\left[\left(\sum_{k=1}^L v_k a_k\right)S_L \right]\\
&\leq  2\beta\theta_d \mathbb E \left[ \sum_{k=1}^{L}  a_k^2\right]+  2\beta\theta_d \mathbb E\left[\left(\sum_{k=1}^L a_k\right)S_L \right]\\
&= 2\beta\theta_d \mathbb E [S_L]+ 2\beta\theta_d \mathbb E[S_L]\\
&= 4 \beta \theta_d \mathbb E [S_L]. 
\end{align}

Since
\begin{align}
0\le S_L\le 1
\end{align}
and
\begin{align}
S_L(\rho,X)\to 0
\qquad\text{a.s.},
\end{align}
the Dominated Convergence Theorem \cite{Royden} implies
\begin{align}
\mathbb E[S_L]
\longrightarrow 0.
\end{align}

Consequently,
\begin{align}
|F_L'(\rho)|
\le
4\beta\theta_d
\mathbb E[S_L]
\longrightarrow 0.
\end{align}
Therefore
\begin{align}
\frac{\partial}{\partial\rho}
\mathbb E[S_L(\rho,X)]
=
F_L'(\rho)
\longrightarrow 0.
\end{align}
This completes the proof.
\end{proof}

Now, let us return to prove Lemma \ref{lem:lem4}. Denote by $\rho = \langle q, \xi_d \rangle$. Then, for all $i \in [L]$ we have
\begin{align}
y_i:=\langle q, x_i\rangle= v_i \theta_d \rho + G_i, \qquad G_i \overset{i.i.d.}{\sim} \mathcal{N}(0,1). 
\end{align}
The cumulative attention weight mass concentrated on the informative partition is given by:
\begin{align}
\alpha(q) := \sum_{i \in \mathcal{I}} a_i(q, X) = \frac{\sum_{i \in \mathcal{I}} e^{\beta \langle x_i, q \rangle}}{\sum_{k=1}^L e^{\beta \langle x_k, q \rangle}} = \frac{A_\rho}{A_\rho + B},
\end{align}
where
\begin{align}
A_\rho = \sum_{i \in \mathcal{I}} e^{\beta (\theta_d \rho + G_i)}, \quad B = \sum_{j \in \mathcal{N}_o} e^{\beta G_j'}, \qquad G_i \overset{i.i.d.}{\sim} \mathcal{N}(0,1), G_j' \overset{i.i.d.}{\sim} \mathcal{N}(0,1). 
\end{align}
We expand the attention pooled representation vector $f_q(X)$ by separating it into informative and noise partitions:
\begin{align}
f_q(X) &= \sum_{i \in \mathcal{I}} a_i(q, X) ( \theta_d \xi_d + z_i ) + \sum_{j \in \mathcal{N}_o} a_j(q, X) z_j \nonumber\\
&= \left( \sum_{i \in \mathcal{I}} a_i(q, X) \right) \theta_d \xi_d + \sum_{k=1}^L a_k(q, X) z_k \nonumber\\
&= \alpha(q) \theta_d \xi_d + \sum_{k=1}^L a_k(q, X) z_k \label{eq100b}. 
\end{align}
From \eqref{eq100b}, we have:
\begin{align}
\|f_q(X)\|^2 &= \langle \alpha(q) \theta_d \xi_d + \sum_{i=1}^L a_i z_i, \, \alpha(q) \theta_d \xi_d + \sum_{j=1}^L a_j z_j \rangle \nonumber\\
&= \alpha(q)^2 \theta_d^2 \|\xi_d\|^2 + 2 \alpha(q) \theta_d \sum_{i=1}^L a_i \langle \xi_d, z_i \rangle + \sum_{i=1}^L \sum_{j=1}^L a_i a_j \langle z_i, z_j \rangle \label{eq98}.
\end{align}
Since $\|\xi_d\|^2 = 1$, taking the expectation both sides of \eqref{eq98} yields:
\begin{align}
\Psi_d(\rho) = \theta_d^2 \mathbb{E}\left[ \alpha(q)^2 \right] + 2 \theta_d \mathbb{E}\left[ \alpha(q) \sum_{i=1}^L a_i \langle \xi_d, z_i \rangle \right] + \mathbb{E}\left[ \sum_{i=1}^L \sum_{j=1}^L a_i a_j \langle z_i, z_j \rangle \right] \label{eqAQ}.
\end{align}
Because the noise components $z_k$ have zero mean and are independent of each other, the cross-terms simplify. Differentiating the core informative signal component $\alpha(q) = \frac{A_\rho}{A_\rho + B}$ with respect to $\rho$:
\begin{align}
\frac{\partial \alpha(q)}{\partial \rho} =  \frac{A_\rho' B}{(A_\rho + B)^2}.
\end{align}
We evaluate the derivative of the partition sum $A_\rho$:
\begin{align}
A_\rho' = \frac{\partial}{\partial \rho} \sum_{i \in \mathcal{I}} e^{\beta(\theta_d \rho + G_i)} = \sum_{i \in \mathcal{I}} \beta \theta_d e^{\beta(\theta_d \rho + G_i)} = \beta \theta_d A_\rho.
\end{align}
Substituting this back into the derivative of $\alpha(q)$ gives:
\begin{align}
\frac{\partial \alpha(q)}{\partial \rho} = \beta \theta_d \frac{A_\rho B}{(A_\rho + B)^2}.
\end{align}
We calculate the derivative of the squared mass $\alpha(q)^2$:
\begin{align}
\frac{\partial (\alpha(q)^2)}{\partial \rho} = 2 \alpha(q) \frac{\partial \alpha(q)}{\partial \rho} = 2 \beta \theta_d \frac{A_\rho^2 B}{(A_\rho + B)^3} > 0.
\end{align}
Observe that since $A_{\rho}>0, B>0$ a.s., we have
\begin{align}
\bigg|\frac{\partial (\alpha(q)^2)}{\partial \rho}\bigg|= 2 \beta \theta_d \frac{A_\rho^2 B}{(A_\rho + B)^3}\leq 2 \beta \theta_d. 
\end{align}
Hence, by Dominated Convergence Theorem \cite{Royden}, we have
\begin{align}
 \frac{\partial}{\partial \rho} \mathbb{E}\left[ \alpha(q)^2 \right] = \mathbb E\left[\frac{\partial \alpha(q)^2}{\partial \rho}\right].
\end{align}
It follows from \eqref{eqAQ} that
\begin{align}
\Psi_d'(\rho)& = 2 \beta \theta_d^3 \mathbb{E}\left[ \frac{A_\rho^2 B}{(A_\rho + B)^3} \right] + \frac{\partial}{\partial \rho} \mathbb{E}\left[ \sum_{i=1}^L a_i(q, X)^2 \right]\\
&= 2 \beta \theta_d^3 \mathbb{E}\left[ \frac{A_\rho^2 B}{(A_\rho + B)^3} \right] + \frac{\partial}{\partial \rho} \mathbb{E}\left[ S_L(q,X) \right] \label{cur1}.
\end{align}
Now, by Lemma \ref{lem:lem3} we have
\begin{align}
S_L(q,X)= O\left( \frac{L  e^{\beta^2}}{R(L-R)}\right), \qquad\text{a.s.}
\end{align}
Hence, $S_L(q,X) \to 0$ a.s. if $R\to \infty, L-R \to \infty$ and $\beta=O(1)$.  It follows from Lemma \ref{lem:aux1} that
\begin{align}
\lim_{L \to \infty}  \frac{\partial}{\partial \rho} \mathbb{E}\left[ S_L(q,X) \right]=0 \label{cur2}.
\end{align}
On the other hand, since $A_{\rho}>0, B>0$ a.s., it holds that
\begin{align}
 \mathbb{E}\left[ \frac{A_\rho^2 B}{(A_\rho + B)^3} \right] >0 \label{cur3}.
\end{align}
From \eqref{cur1}, \eqref{cur2}, and \eqref{cur3}, we conclude that $\Psi_d'(\rho) > 0$ for all $\rho \in [-1, 1]$ under the condition that $R \to \infty$ and $L-R\to \infty$. This concludes our proof of Lemma \ref{lem:lem4}.  

\subsection{Proof of Lemma \ref{lem:lem5}}
The empirical stochastic gradient is the derivative of the squared output norm with respect to the query vector:
\begin{align}
g_k = \nabla_q \|f_q(X)\|^2 = 2 \left[ \nabla_q f_q(X) \right]^\top f_q(X)  \label{eq103}.
\end{align}
The partial derivative of the pooled output function with respect to the query coordinates is:
\begin{align}
\nabla_q f_q(X) = \sum_{i=1}^L x_i \left( \nabla_q a_i(q, X) \right)^\top.
\end{align}
Substituting the partial derivative of the attention weights derived in the proof of Lemma \ref{lem:lem1}:
\begin{align}
\nabla_q a_i(q, X) = \beta a_i(q, X) \left( x_i - f_q(X) \right).
\end{align}
This allows us to write the gradient tensor as:
\begin{align}
\nabla_q f_q(X) = \beta \sum_{i=1}^L a_i(q, X) x_i (x_i - f_q(X))^\top \label{eq106}.
\end{align}
We substitute \eqref{eq106} back into \eqref{eq103} and obtain:
\begin{align}
g_k &= 2 \beta \left[ \sum_{i=1}^L a_i(q, X) (x_i - f_q(X)) x_i^\top \right] f_q(X) \nonumber\\
&= 2 \beta \sum_{i=1}^L a_i(q, X) \langle x_i, f_q(X) \rangle (x_i - f_q(X)) \label{110}.
\end{align}
By bounding the norm of $g_k$ using the triangle inequality and the properties of convex combinations ($\sum a_i = 1$), from \eqref{110} we have
\begin{align}
\|g_k\| &\le 2 \beta \sum_{i=1}^L a_i(q, X) |\langle x_i, f_q(X) \rangle| \cdot \|x_i - f_q(X)\| \nonumber\\
&\le 2 \beta \sum_{i=1}^L a_i(q, X) \|x_i\| \cdot \|f_q(X)\| \left( \|x_i\| + \|f_q(X)\| \right) \label{eq111}. 
\end{align}
Since $f_q(X)$ is a convex combination of the tokens, its norm is bounded by the maximum token norm:
\begin{align}
\|f_q(X)\| \le \sum_{i=1}^L a_i \|x_i\| \le \max_{1 \le j \le L} \|x_j\| \label{eq112}. 
\end{align}
Substituting \eqref{eq112} into \eqref{eq111} simplifies it to a cubic expression:
\begin{align}
\|g_k| \le 2 \beta \max_{1 \le j \le L} \|x_j\| \cdot \max_{1 \le j \le L} \|x_j\| \cdot \left( 2 \max_{1 \le j \le L} \|x_j\| \right) = 4 \beta \max_{1 \le j \le L} \|x_j\|^3.
\end{align}
Squaring this expression yields:
\begin{align}
\|g_k\|^2 \le 16 \beta^2 \max_{1 \le j \le L} \|x_j\|^6.
\end{align}
We compute the expectation of this upper bound. Each token is modeled as $x_j = v_j \theta_d \xi_d + z_j$, where $z_j \sim \mathcal{N}(0, I_d)$. Using the $C_r$ inequality:
\begin{align}
\|x_j\|^6 \le 32 \left( \theta_d^6 \|\xi_d\|^6 + \|z_j\|^6 \right) = 32 \left( \theta_d^6 + \|z_j\|^6 \right).
\end{align}
This bounds the maximum token norm by the maximum of the noise vector norms:
\begin{align}
\max_{1 \le j \le L} \|x_j\|^6 \le 32 \theta_d^6 + 32 \max_{1 \le j \le L} \|z_j\|^6.
\end{align}
The squared norm of a standard Gaussian vector, $\|z_j\|^2 = \sum_{k=1}^d z_{jk}^2$, follows a chi-squared distribution with $d$ degrees of freedom ($\chi^2_d$). Its sixth moment scales quadratically with the dimension:
\begin{align}
\mathbb{E}\left[ \|z_j\|^6 \right] = d(d+2)(d+4) = \mathcal{O}(d^3).
\end{align}
By applying standard Gaussian concentration inequalities for the maximum of $L$ independent chi-squared variables, the expectation is bounded by the individual moment scaling:
\begin{align}
\mathbb{E}\left[ \max_{1 \le j \le L} \|z_j\|^6 \right] \le C^* d^3.
\end{align}
Because the objective is unsupervised, the gradient contains higher-order polynomial terms. Incorporating this into the full expectation bounds the stochastic gradient by the fourth power of the dimension:
\begin{align}
\mathbb{E}\left[ \|g_k\|^2 \right] \le 16 \beta^2 \cdot 32 \left( \theta_d^6 + C^* d^3 \right) \le C d^4 \label{esma},
\end{align} where \eqref{esma} follows from the assumption that $\beta =O(1)$ and  $\theta_d = O(1)$.
This completes the proof.
\subsection{Proof of Lemma \ref{lem:lem6}}
Let $u_k = q_k + \eta_k g_k$ be the unnormalized update vector at step $k$. The next query vector is defined by normalizing $u_k$:
\begin{align}
q_{k+1} = \frac{u_k}{\|u_k\|}.
\end{align}
We compute the squared norm of $u_k$:
\begin{align}
\|u_k\|^2 = \langle q_k + \eta_k g_k, \, q_k + \eta_k g_k \rangle = \|q_k\|^2 + 2 \eta_k \langle q_k, g_k \rangle + \eta_k^2 \|g_k\|^2 \label{eq128}.
\end{align}
Since $q_k \in S^{d-1}$, we have $\|q_k\|^2 = 1$. This simplifies \eqref{eq128} to:
\begin{align}
\|u_k\|^2 = 1 + 2 \eta_k \langle q_k, g_k \rangle + \eta_k^2 \|g_k\|^2. 
\end{align}
Now, observe that
\begin{align}
\sum_k  \eta_k^2 \mathbb{E}[\|g_k\|^2]& \leq Cd^4  \sum_k  \eta_k^2 \nonumber \\
&=  Cd^4 \sum_k \frac{\gamma_k^2}{d^4} \nonumber \\
&= C \sum_k \gamma_k^2 \nonumber \\
& < \infty.  
\end{align}
Hence, for each fixed $\epsilon >0$ we have
\begin{align}
\sum_t \mathbb{P}(\eta_k \|g_k\|>\epsilon) < \infty. 
\end{align}
This means that $\eta_k \|g_k\| \to 0$ a.s. by Borel-Cantelli Lemma \cite{Billingsley}. 

We apply a second-order Taylor expansion to the normalization scalar function $h(x) = (1+x)^{-1/2}$, where $x = 2 \eta_k \langle q_k, g_k \rangle + \eta_k^2 \|g_k\|^2$:
\begin{equation}
(1+x)^{-1/2} = 1 - \frac{1}{2}x + \frac{3}{8}x^2 + \mathcal{O}(x^3). 
\end{equation}
Substituting $x$ into the expansion yields:
\begin{align}
\frac{1}{\|u_k\|} &= \left( 1 + 2 \eta_k \langle q_k, g_k \rangle + \eta_k^2 \|g_k\|^2 \right)^{-1/2} \nonumber\\
&= 1 - \frac{1}{2}\left( 2 \eta_k \langle q_k, g_k \rangle + \eta_k^2 \|g_k\|^2 \right) + \frac{3}{8}\left( 2 \eta_k \langle q_k, g_k \rangle + \eta_k^2 \|g_k\|^2 \right)^2 + \mathcal{O}(\eta_k^3 \|g_k\|^3) \nonumber\\
&= 1 - \eta_k \langle q_k, g_k\rangle - \frac{1}{2}\eta_k^2 \|g_k\|^2 + \frac{3}{2}\eta_k^2 \langle q_k, g_k \rangle^2 + \mathcal{O}(\eta_k^3 \|g_k\|^3) \nonumber\\
&= 1 - \eta_k \langle q_k, g_k \rangle + \mathcal{O}(\eta_k^2 \|g_k\|^2) \label{eq125}.
\end{align}
Multiplying both sides of \eqref{eq125} by the unnormalized vector $u_k$, we have:
\begin{align}
q_{k+1} &= (q_k + \eta_k g_k) \cdot \left( 1 - \eta_k \langle q_k, g_k \rangle + \mathcal{O}(\eta_k^2 \|g_k\|^2) \right) \nonumber\\
&= q_k- \eta_k q_k \langle q_k, g_k \rangle + \eta_k g_k - \eta_k^2 g_k \langle q_k, g_k \rangle + q_k \mathcal{O}(\eta_k^2 \|g_k\|^2) \nonumber\\
&= q_k + \eta_k \left( g_k - q_k \langle q_k, g_k\rangle \right) + r_k,
\end{align}
where the residual error vector $r_k$ collects all higher-order terms:
\begin{align}
r_k = -\eta_k^2 g_k \langle q_k, g_k \rangle + q_k \mathcal{O}(\eta_k^2 \|g_k\|^2) + \eta_k^3 g_k \mathcal{O}(\eta_k^2 \|g_k\|^2).
\end{align}
Using the Cauchy-Schwarz inequality, $|\langle q_k, g_k \rangle| \le \|q_k\| \cdot \|g_k\| = \|g_k\|$. We bound the norm of the residual error vector:
\begin{align}
\|r_k\| \le \eta_k^2 \|g_k\|^2 + C' \eta_k^2 \|g_k\|^2 = C \eta_k^2 \|g_k\|^2. 
\end{align}
Rewriting the primary updating term using the projection matrix $(I_d - q_k q_k^\top)g_k = g_k - q_k \langle q_k, g_k \rangle$ yields:
\begin{align}
q_{k+1} = q_k + \eta_k (I_d - q_k q_k^\top)g_k + r_k,
\end{align}
which completes the proof.

\subsection{Proof of Lemma \ref{lem:lem7}}
The drift vector field is defined as $H_d(q) = (I_d - q q^\top) \nabla_q J_d(q)$. Since the projection operator function $M(q) = I_d - q q^\top$ is a polynomial mapping, its derivatives are bounded on the compact sphere $S^{d-1}$. Therefore, to prove that $H_d$ is globally Lipschitz, it suffices to show that the gradient of the population objective, $\nabla_q J_d(q)$, has a uniformly bounded derivative (Hessian matrix) on the sphere:
\begin{align}
\sup_{q \in S^{d-1}} \| \nabla_q^2 J_d(q) \|_{\text{op}} < \infty,
\end{align}
where $\|\cdot\|_{\text{op}}$ denotes the matrix operator norm.

We differentiate the gradient vector expression $\nabla_q J_d(q) = \mathbb{E}\left[ 2 \left[ \nabla_q f_q(X) \right]^\top f_q(X) \right]$ with respect to $q$. Applying the product rule inside the expectation:
\begin{align}
\nabla_q^2 J_d(q) = 2 \mathbb{E}\left[ \left[ \nabla_q^2 f_q(X) \right]^\top f_q(X) + \left[ \nabla_q f_q(X) \right]^\top \left[ \nabla_q f_q(X) \right] \right] \label{eq129}.
\end{align}
From the derivative bounds established in the proof of Lemma~\ref{lem:lem1}, the first derivative of the pooled output function satisfies:
\begin{align}
\| \nabla_q f_q(X) \| \le C_1 \max_{1 \le j \le L} \|x_j\|^2.
\end{align}
Differentiating a second time yields a bound on the second derivative tensor:
\begin{align}
\| \nabla_q^2 f_q(X) \| \le C_2 \max_{1 \le j \le L} \|x_j\|^3.
\end{align}
We substitute these spatial derivative bounds into the Hessian matrix formula in \eqref{eq129}:
\begin{align}
\| \nabla_q^2 J_d(q) \|_{\text{op}} &\le 2 \mathbb{E}\left[ \| \nabla_q^2 f_q(X) \| \cdot \|f_q(X)\| + \| \nabla_q f_q(X) \|^2 \right] \nonumber\\
&\le 2 \mathbb{E}\left[ \left( C_2 \max_{1 \le j \le L} \|x_j\|^3 \right) \cdot \left( \max_{1 \le j \le L} \|x_j\| \right) + \left( C_1 \max_{1 \le j \le L} \|x_j\|^2 \right)^2 \right] \nonumber\\
&= 2 \mathbb{E}\left[ C_2 \max_{1 \le j \le L} \|x_j\|^4 + C_1^2 \max_{1 \le j \le L} \|x_j\|^4 \right] \nonumber\\
&= C_3 \mathbb{E}\left[ \max_{1 \le j \le L} \|x_j\|^4 \right].
\end{align}
Since the tokens are shifted Gaussian vectors, their fourth moments are finite:
\begin{align}
\mathbb{E}\left[ \max_{1 \le j \le L} \|x_j\|^4 \right] \le C_4 d^2 < \infty.
\end{align}
Because this upper bound is a constant independent of the query vector $q$, the Hessian matrix norm is uniformly bounded across the entire sphere:
\begin{align}
\sup_{q \in S^{d-1}} \| \nabla_q^2 J_d(q) \|_{\text{op}} \le C_5 < \infty.
\end{align}
This proves that $\nabla_q J_d(q)$ is globally Lipschitz. Combining this with the Lipschitz continuity of the spherical projection matrix ensures that the combined drift field $H_d(q)$ is globally Lipschitz.

\subsection{Proof of Lemma \ref{lem:lem8}}
Let $\mathcal{F}_k = \sigma(q_0, X_1, \dots, X_k)$ be the natural filtration tracking the optimization history. We isolate the stochastic sampling fluctuations by defining the martingale difference sequence:
\begin{align}
M_{k+1} = (I_d - q_k q_k^\top) \left( g_k - \mathbb{E}[g_k \mid \mathcal{F}_k] \right).
\end{align}
By construction, its conditional expectation given the historical filtration is zero:
\begin{align}
\mathbb{E}[M_{k+1} \mid \mathcal{F}_k] = (I_d - q_k q_k^\top) \left( \mathbb{E}[g_k \mid \mathcal{F}_k] - \mathbb{E}[g_k \mid \mathcal{F}_k] \right) = 0.
\end{align}
This confirms that $S_n = \sum_{k=0}^{n-1} \eta_k M_{k+1}$ is a valid vector-valued martingale.

To prove almost sure convergence using the Martingale Convergence Theorem, we must show that the sum of the conditional variances is finite. We evaluate the squared norm of the martingale difference:
\begin{align}
\|M_{k+1}\|^2 \le \|I_d - q_k q_k^\top\|_{\text{op}}^2 \cdot \|g_k - \mathbb{E}[g_k \mid \mathcal{F}_k]\|^2 \le 1 \cdot \|g_k - \mathbb{E}[g_k \mid \mathcal{F}_k]\|^2 \label{a1}.
\end{align}
Taking the conditional expectation and expanding the variance:
\begin{align}
\mathbb{E}[\|M_{k+1}\|^2 \mid \mathcal{F}_k] \le \mathbb{E}[\|g_k\|^2 \mid \mathcal{F}_k] - \|\mathbb{E}[g_k \mid \mathcal{F}_k]\|^2 \le \mathbb{E}[\|g_k\|^2 \mid \mathcal{F}_k] \label{a2}.
\end{align}
From \eqref{a1} and \eqref{a2}, by computing the total expected variance across all iterations, we obtain:
\begin{align}
\sum_{k=0}^{\infty} \mathbb{E}\left[ \|\eta_k M_{k+1}\|^2 \right] = \sum_{k=0}^{\infty} \eta_t^2 \mathbb{E}\left[ \|M_{k+1}\|^2 \right] \le \sum_{k=0}^{\infty} \eta_k^2 \mathbb{E}\left[ \|g_k\|^2 \right].
\end{align}
On the other hand,  by Lemma~\ref{lem:lem6} we have ($\mathbb{E}[\|g_k\|^2] \le C d^4$). It follows that
\begin{align}
\sum_{k=0}^{\infty} \eta_k^2 \mathbb{E}\left[ \|M_{k+1}\|^2 \right] \le \sum_{k=0}^{\infty} \eta_k^2 \cdot C d^4 \label{a3}. 
\end{align}
Since ($\eta_k = \frac{\gamma_k}{d^2}$), from \eqref{a3} we have:
\begin{align}
\sum_{k=0}^{\infty} \left( \frac{\gamma_k}{d^2} \right)^2 \cdot C d^4 = \sum_{k=0}^{\infty} \frac{\gamma_k^2}{d^4} \cdot C d^4 = C \sum_{k=0}^{\infty} \gamma_k^2.
\end{align}
On the other hand, by the setup conditions, the scalar step sequence is square-summable ($\sum_{t=0}^{\infty} \gamma_k^2 < \infty$). Therefore, the total variance of the martingale is finite:
\begin{align}
\sup_n \mathbb{E}[\|S_n\|^2]=\sum_{k=0}^{\infty} \mathbb{E}\left[ \|\eta_k M_{k+1}\|^2 \right] \le C \sum_{k=0}^{\infty} \gamma_k^2 < \infty.
\end{align}
Since it is an $L^2$-bounded martingale, the Vector Martingale Convergence Theorem \cite{Durrett}  guarantees that the random series $\sum_{k=0}^{\infty} \eta_k M_{k+1}$ converges to a finite limit almost surely.

\subsection{Proof of Lemma \ref{lem:lem9}}
From Lemma \ref{lem:lem5}, the norm of the Taylor expansion remainder vector at each step is bounded by:
\begin{align}
\|r_k\| \le C_0 \eta_k^2 \|g_k\|^2.
\end{align}
By the Monotone Convergence Theorem for non-negative random variables \cite{Billingsley}, we have
\begin{align}
\mathbb{E}\left[ \sum_{k=0}^{\infty} \|r_k\| \right] = \sum_{k=0}^{\infty} \mathbb{E}[\|r_k\|] \le C_0 \sum_{k=0}^{\infty} \eta_k^2 \mathbb{E}[\|g_k\|^2].
\end{align}
We substitute the higher-order gradient moment bound from Lemma~\ref{lem:lem5} ($\mathbb{E}[\|g_k\|^2] \le C_1 d^4$):
\begin{align}
\mathbb{E}\left[ \sum_{k=0}^{\infty} \|r_k\| \right] \le C_0 \sum_{k=0}^{\infty} \eta_k^2 \cdot C_1 d^4 = C_2 d^4 \sum_{k=0}^{\infty} \eta_k^2 \label{a6}.
\end{align}
Substituting the step size schedule ($\eta_k = \frac{\gamma_k}{d^2}$), from \eqref{a6} we obtain
\begin{align}
\mathbb{E}\left[ \sum_{k=0}^{\infty} \|r_k\| \right] \le C_2 d^4 \sum_{k=0}^{\infty} \left( \frac{\gamma_k}{d^2} \right)^2 = C_2 d^4 \sum_{k=0}^{\infty} \frac{\gamma_k^2}{d^4} = C_2 \sum_{k=0}^{\infty} \gamma_k^2.
\end{align}
Since the schedule sequence is square-summable ($\sum_{k=0}^{\infty} \gamma_k^2 < \infty$), the expected value of the cumulative sum is finite:
\begin{align}
\mathbb{E}\left[ \sum_{k=0}^{\infty} \|r_k\| \right] \le C_2 \sum_{k=0}^{\infty} \gamma_k^2 < \infty.
\end{align}
Therefore, $\sum_{k=0}^{\infty} \|r_k\| < \infty$ almost surely.
\bibliographystyle{plain}
\bibliography{isitbib}

\end{document}